\documentclass[12pt]{article}
\usepackage{amssymb}
\usepackage{amsmath}
\usepackage{amsthm}
\usepackage{graphicx}
\usepackage{setspace}
\usepackage{natbib}
\usepackage{bbold}

\usepackage{algorithm}
\usepackage{algorithmic}

\newcommand{\hmu}{\hat{\mathbf{\mu}}}

\newcommand{\veps}{\varepsilon}
\newcommand{\tm}{t_{\rm mix}}

\renewcommand{\S}{\mathcal S}
\newcommand{\A}{\mathcal A}
\newcommand{\K}{\mathcal K}

\newcommand{\bmu}{\boldsymbol{\mu}}
\newcommand{\bhmu}{\hat{\boldsymbol{\mu}}}

\newcommand{\br}{\mathbf{r}}
\newcommand{\bhr}{\hat{\mathbf{r}}}

\newcommand{\osp}{\textsc{Osp}}

\newtheorem{lemma}{Lemma}
\newtheorem{corollary}{Corollary}
\newtheorem{theorem}{Theorem}

\begin{document}
\begin{titlepage}
 \title{\Large\bf Regret Bounds for Reinforcement Learning via Markov Chain Concentration}
\author{{\bf Ronald ORTNER}\\
 \small Montanuniversit\"at Leoben, Austria\\\vspace{-0.1cm}
 \small e-mail: rortner@unileoben.ac.at}

\end{titlepage}

\sloppy


\maketitle



\textit{To the memory of Thomas Jaksch}
\bigskip

 \begin{abstract}
We give a simple optimistic algorithm for which it is easy to derive regret bounds of $\tilde{O}(\sqrt{\tm SAT})$
after $T$ steps in uniformly ergodic Markov decision processes with $S$ states, $A$ actions, and mixing time parameter $\tm$.
These bounds are the first regret bounds in the general, non-episodic setting with an optimal dependence on all given parameters. 
They could only be improved by using an alternative mixing time parameter.
 \end{abstract} 

\section{Introduction}

Starting with \cite{buka}, regret bounds for reinforcement learning have addressed the question of how difficult it is
to learn optimal behavior in an unknown Markov decision process (MDP). Some of these bounds like the one derived in the
mentioned \cite{buka} depend on particular properties of the underlying MDP, typically some kind of gap that specifies
the distance between an optimal and a sub-optimal action or policy (see e.g.\ \cite{ok} for a recent refinement of such
bounds). The first so-called problem independent bounds that have no dependence on any gap-parameter were obtained in
\cite{ucrl}. For MDPs with $S$ states, $A$ actions and diameter $D$ the regret of the UCRL algorithm was shown to be
$\tilde{O}(DS\sqrt{AT})$ after any $T$ steps. A corresponding lower bound of $\Omega(\sqrt{DSAT})$ left the open
question of the true dependence of the regret on the parameters $S$ and $D$.
Recently, regret bounds of $\tilde{O}(D\sqrt{SAT})$ have been claimed in \cite{shipra}, however there seems to be a gap in the proof, cf. Sec.~38.9 of \cite{torcsa}, so that the original bounds of \cite{ucrl} are still the best known bounds. 

In the simpler episodic setting, the gap between upper and lower bounds has been closed in~\cite{azar}, showing that the
regret is of order $\tilde{O}(\sqrt{HSAT})$, where $H$ is the length of an episode. However, while bounds for the
non-episodic setting can be easily transferred to the episodic setting, the reverse is not true. We also note that
another kind of regret bounds that appears in the literature assumes an MDP sampled from some distribution (see e.g.\
\cite{osband} for a recent contribution). Regret bounds in this Bayesian setting cannot be turned into bounds for the
worst case setting as considered here.

There is also quite some work on bounds on the number of samples from a generative model necessary to approximate the optimal policy by an error of at most $\varepsilon$.
Obviously, having access to a generative model makes learning the optimal policy easier than in the online setting
considered here. However, for ergodic MDPs it could be argued that any policy reaches any state so that in this case
sample complexity bounds could in principle be turned into regret bounds. We first note that this seems difficult for
bounds in the discounted setting, which make up the majority in the literature. Bounds in the discounted setting (see
e.g.\ \cite{pac} or \cite{sidford} for a more recent contribution obtaining near-optimal bounds) depend on the term
$1-\gamma$, where $\gamma$ is the discount factor, and it is not clear how this term translates into a mixing time
parameter in the average reward case. For the few results in the average reward setting the best sample complexity bound
we are aware of is the bound of $\tilde{O}\big(\frac{\tau^2 \tm^2 SA}{\varepsilon^2}\big)$ of \cite{wang}, where $\tm$
is a mixing time parameter like ours (cf.\ below) and $\tau$ characterizes the range of stationary distributions across
policies. Translated into respective regret bounds, these would have a worse (i.e., linear) dependence on the mixing
time and would depend on the additional parameter $\tau>1$, which does not appear in our bounds.

Starting with \cite{e3,rmax} there are also sample complexity bounds in the literature that were derived for settings without generative sampling model. Although this is obviously harder, there are bounds for the discounted case where the dependence with respect to $S$, $A$, and $\varepsilon$ is the same as for the case with a generative sampling model \cite{szita}.
However, we are not aware of any such bounds for the undiscounted setting that would translate into online regret bounds optimal in $S$, $A$, and $T$.
\\

In this note, we present a simple algorithm that allows the derivation of regret bounds of $\tilde{O}(\sqrt{\tm SAT})$
for uniformly ergodic MDPs with mixing time $\tm$, a parameter that measures how long it takes to approximate the
stationary distribution induced by any policy. These bounds are optimal with respect to the parameters $S$, $A$, $T$,
and $\tm$. The only possible improvement is a replacement of $\tm$ by a parameter that may be smaller for some MDPs,
such as the diameter~\cite{ucrl} or the bias span \cite{regal,ronan}. We note, however, that it is easy to give MDPs for
which $\tm$ is basically of the same size as the mentioned alternative parameters.\footnote{See~\cite{ucrl,regal} for a
discussion of various transition parameters used in the literature.} Accordingly, the obtained bound basically closes
the gap between upper and lower bound on the regret for a subclass of MDPs.

Algorithmically, the algorithm we propose works like an optimistic bandit algorithm such as UCB~\cite{acbf}.
Such algorithms have been proposed before for MDP settings with a limited set of policies~\cite{poladv}.
The main difference to the latter approach is that due to the re-use of samples we obtain regret bounds that do not scale with 
the number of policies but with the number of state-action pairs. We note however that as \cite{poladv} our algorithm needs to evaluate each policy independently, which makes it impractical. The proof of the regret bound is much simpler than for bounds achieved before and relies on concentration results for Markov chains.

\section{Setting}

We consider reinforcement learning in an average reward \textit{Markov decision process} (\textit{MDP}) with finite state space $\S$ and finite action space $\A$.
We assume that each stationary policy $\pi:\S\to \A$ induces a uniformly ergodic\footnote{See Section~\ref{sec:prel} for definitions.} Markov chain on the state space.
In such MDPs, which we call \textit{uniformly ergodic}, the chain induced by a policy $\pi$ has a unique stationary distribution $\mu_\pi$,
and the (state-independent) average reward $\rho_{\pi}$ can be written as
$\rho_{\pi}=\bmu_{\pi}^\top \br_{\pi}$, where $\bmu_{\pi}=(\mu_\pi(s))_s$ and $\br_\pi=(r(s,\pi(s))_s$ are 
the (column) vectors for the stationary distribution and the average reward under $\pi$, respectively.
We assume that the reward distribution for each state-action pair $(s,a)$ has support in $[0,1]$.

The maximal average reward is known (cf.~\cite{puterman}) to be achieved by a stationary policy $\pi^*$ that gives average reward $\rho^*:=\rho_{\pi^*}$.
We are interested in the \textit{regret} accumulated by an algorithm after any number of $T$ steps defined as\footnote{Since we are only interested in upper bounds on this quantity we ignore the dependence on the initial state to keep things simpler. See \cite{ucrl} for a discussion.}
\[
   R_T := T\rho^* - \sum_t r_t,
\]
where $r_t$ are the (random) rewards collected by the algorithm at each step $t$.

\section{Preliminaries on Markov Chains}\label{sec:prel}
In this section, we give some definitions and results about Markov chain concentration that we will use in the following. 
\subsection{Mixing Times}
For two distributions $P,Q$ over the same state space $(\S,\mathcal{F})$ with $\sigma$-algebra~$\mathcal{F}$, let
\[
   d_{TV}(P,Q) := \sup_{A\in\mathcal{F}} |P(A)-Q(A)|
\]
be the \textit{total variational distance} between $P$ and $Q$. 
A Markov chain with a transition kernel $p$ and a stationary distribution $\mu$ is said to be \textit{uniformly ergodic}, if there are a $\theta<1$ and a finite $L$ such that
\[
\sup_{s\in\S} d_{TV}(p^n(s,\cdot),\mu) \leq L \theta^n.
\]
Furthermore, the \textit{mixing time} $\tm$ of the Markov chain is defined as 
\[
    \tm := \min\big\{n\,|\, \sup_{s\in \S} d_{TV}(p^n(s,\cdot),\mu) \leq \tfrac14 \big\}.
\]
For a uniformly ergodic MDP we set the mixing time $\tm^\pi$ of a policy $\pi$ to be the mixing time of the Markov chain induced by $\pi$, and define the \textit{mixing time of the MDP} to be $\tm:=\max_\pi \tm^\pi$.\bigskip

\subsection{McDiarmid's Inequality for Markov Chains}
Our results mainly rely on the following version of McDiarmid's inequality for Markov chains from \cite{paulin}.
\begin{lemma}\label{lem:mcd}(Corollary 2.10  and the following Remark 2.11 of \cite{paulin})\\
Consider a uniformly ergodic Markov chain $X_1,\ldots, X_n$ with state space $\mathcal S$ and mixing time $\tm$. 
Let $f:{\mathcal S}^n \to \mathbb R$ with 
\begin{equation}\label{eq:cond}
  f(s_1,\ldots,s_n)-f(s'_1,\ldots,s'_n) \, \leq \, \sum_i c_i \mathbb{1}[s_i\neq s'_i].
\end{equation}
Then 
\[
  \mathbb{P}\Big\{ \big| f(X_1,\ldots,X_n) - \mathbb{E}[f(X_1,\ldots,X_n)] \big| \geq \veps \Big\}  
        \leq  2 \exp\left( - \frac{2\veps^2}{9 \, \|c\|_2^2 \, \tm} \right) .
\]
\end{lemma}

Lemma \ref{lem:mcd} can be used to obtain a concentration result for the empirical average reward of any policy $\pi$ in an MDP.
This works analogously to the concentration bounds for the total variational distance between the empirical and the stationary distribution (Proposition 2.18 in \cite{paulin}).
 
 \begin{corollary}\label{corn}
Consider an MDP and a policy $\pi$ that induces a uniformly ergodic Markov chain with mixing time $\tm$.
Using (column) vector notation $\bmu:=(\mu_{\pi}(s))_s$ and $\br:=(r(s,\pi(s))_s$ for the stationary distribution and the reward function under $\pi$, and writing $\bhmu^n$ for the empirical distribution after $n$ steps defined as $\hmu^n(s):=\frac{1}{n}\sum_{i=1}^n \mathbb{1}\{X_i=s\}$, it holds that
\[
   \mathbb{P}\Big\{  \big| \bhmu^{n\top} \br - \bmu^\top \br  \big| \geq \veps \Big\}  
         \leq  2 \exp\left( - \frac{2\veps^2 n}{9 \tm} \right) .
 \]
 \end{corollary}
 
 \begin{proof}
 Setting $f(X_1,\ldots,X_n):=\frac{1}{n} \big(r(X_1,\pi(X_1))+\ldots + r(X_n,\pi(X_n))\big)$, condition \eqref{eq:cond} holds choosing $c_i=\frac{1}{n}$ for  $i=1,\ldots,n$ and the claim follows from Lemma~\ref{lem:mcd}.
 \end{proof}

Choosing the error probability to be $\delta$, we obtain the following confidence interval that will be used by our algorithm. 
\begin{corollary}\label{cor-ci}
Using the same assumptions and notation of Corollary \ref{corn}, with probability at least $1-\delta$,
 \[
      \big| \bhmu^{n\top} \br - \bmu^\top \br  \big| \;\leq\; \sqrt{\frac{9 \tm \log \frac{2}{\delta}}{2n}}.
 \]
\end{corollary}

\subsection{Concentration of the Empirical Distribution}
We will also need the following results on the concentration of the empirical state distribution of Markov chains from \cite{paulin}.
In the following, consider a uniformly ergodic Markov chain $X_1,\ldots, X_n$ with a stationary distribution $\mu$ and a mixing time $\tm$. 
Let $\hmu^n$ be the empirical distribution after performing $n$ steps in the chain.

\begin{lemma}\label{lem:mc-conc}(Proposition~2.18 in \cite{paulin})
\[
  \mathbb{P}\Big\{ \big| d_{TV}(\mu,\hmu^n) - \mathbb{E}[d_{TV}(\mu,\hmu^n)] \big| \geq \veps \Big\}  
        \leq  2 \exp\left( - \frac{2\veps^2 n}{9 \tm} \right) .
\]
\end{lemma}

\begin{lemma}\label{lem:exp-d}(Proposition~3.16 and following remark in \cite{paulin})

\[
    \mathbb{E}[d_{TV}(\mu,\hmu^n)]  \leq \sum_{s\in \S} \min\left(\sqrt{\frac{8\mu(s)}{n\beta}},\mu(s)\right),
\]
where $\beta$ is the pseudo-spectral gap\footnote{The pseudo-spectral gap is defined as $\max_k \big\{\frac{\gamma(\mathbf{P}^{*k}\mathbf{P}^k)}{k}\big\}$, where $\mathbf{P}$ is the transition kernel interpreted as linear operator, $\mathbf{P}^*$ is the adjoint of $\mathbf{P}$, and $\gamma(\mathbf{P}^{*k}\mathbf{P}^k)$ is the spectral gap of the self-adjoint operator $\mathbf{P}^{*k}\mathbf{P}^k$. For more details see \cite{paulin}. Here we do not make direct use of this quantity and only use the bound given in Lemma~\ref{lem:psg}.} of the chain.
\end{lemma}
 
\begin{lemma}\label{lem:psg}(Proposition~3.4 in \cite{paulin})
In uniformly ergodic Markov chains, the pseudo-spectral gap $\beta$ can be bounded via the mixing time $\tm$ as
\[
    \tfrac{1}{\beta} \leq 2 \tm.
\]
\end{lemma}

We summarize these results in the following corollary.
\begin{corollary}\label{cor}
 With probability at least $1-\delta$,
\[
   d_{TV}(\mu,\hmu^n) \leq \sqrt{\frac{38 S \tm \log \frac{2}{\delta}}{n}}.
\]
\end{corollary}
\begin{proof}
Using the bound of Lemma~\ref{lem:psg} in Lemma~\ref{lem:exp-d} and setting the error probability in Lemma~\ref{lem:mc-conc} to $\delta$, one obtains by Jensen's inequality
\[
   d_{TV}(\mu,\hmu^n) \leq \sqrt{\frac{16 S \tm \mu(s)}{n}} + \sqrt{\frac{9 \tm \log \frac{2}{\delta}}{2n}},
\]
and the claim of the corollary follows immediately.
\end{proof}

\section{Algorithm}

At the core, the \osp\ algorithm we propose works like the UCB algorithm in the bandit setting.
In our case, each policy corresponds to an arm, and the concentration results of the previous 
chapter are used to obtain suitable confidence intervals for the MDP setting.

\begin{algorithm}[!ht]
\caption{Optimistic Sample Path (\osp)}\label{alg}
\begin{algorithmic}[1]
\STATE \textbf{Input:} confidence $\delta$, horizon $T$, (upper bound on) mixing time $\tm$ \\ \smallskip
  \textit{//Initialization:} 
\STATE Set $t:=1$ and let the sequence $\mathcal O$ of observations $(s,a,r,s')$ be empty.\\ \smallskip
 \textit{// Compute sample paths for policies}
\FOR {phases $k = 1,2,\ldots$}  \label{l:phase}
\FOR {each policy $\pi:\mathcal S \to \mathcal A$} 
\STATE Use Alg.~\ref{alg2} to construct a non-extendible sample path $\mathcal P_\pi$ from~$\mathcal O$. \label{l:path}
\STATE Let \vspace{-4mm} \label{l:civ}
\[ 
  \hat\rho_\pi:=\tfrac{1}{|\mathcal{P}_\pi|}\hspace{-6mm}\sum_{(s,\pi(s),r,s')\in\mathcal{P}_\pi}\hspace{-6mm} r
  \mbox{,\, and set } 
   \tilde\rho_\pi:=\hat{\rho}_{\pi} + \sqrt{\frac{8\tm\log\frac{8tT}{\delta}}{|\mathcal P_\pi|}}. 
\]\vspace{-4mm}
\ENDFOR\\\smallskip
 \textit{// Choose optimistic policy}
\STATE Choose $\pi_k:=\arg\max_\pi \tilde{\rho}_\pi$ and set $n_{<k}:=|\mathcal P_{\pi_k}|$.\label{l:select}\\\smallskip
 \textit{// Execute optimistic policy ${\pi}_k$}
\FOR {$\tau=1,\ldots,n_k:=\max\Big\{n_{<k},\sqrt{\frac{T}{SA}}\Big\}$} \label{l:act}
\STATE Choose action $a_t = {\pi}_k(s_t)$, obtain reward $r_t$, and observe $s_{t+1}$.  \label{l:obs}\\
Set $t := t + 1$ and append the observation $(s_t,a_t,r_t,s_{t+1})$ to $\mathcal O$.
\ENDFOR
\ENDFOR
\end{algorithmic}
\end{algorithm}

\osp\ (shown in detail as Algorithm 1) does not evaluate the policies at each time step. Instead, it proceeds in phases\footnote{We emphasize that we consider \textit{non-episodic} reinforcement learning and that these phases are internal to the algorithm.} (cf.~line~\ref{l:phase} of \osp), where in each phase~$k$ an optimistic policy $\pi_k$ is selected (line~\ref{l:select}). This is done (cf.~line~\ref{l:path}) by first constructing for each policy~$\pi$ a sample path $\mathcal P_{\pi}=\big((s_t,\pi(s_t),r_t,s_{t+1})\big)_{t=1}^n$ from the observations so far. Accordingly, the algorithm keeps a record of all observations. That is, after choosing in a state~$s$ an action $a$, obtaining the reward $r$, and observing a transition to the next state $s'$, the respective observation $(s,a,r,s')$ is appended to the sequence of observations $\mathcal O$ (cf.~line~\ref{l:obs}). 

The sample path $\mathcal P_\pi$ constructed from the observation sequence $\mathcal O$ contains each observation from $\mathcal O$ at most once. Further, the path $\mathcal P_{\pi}=\big((s_t,\pi(s_t),r_t,s_{t+1})\big)_{t=1}^n$ is such that there is no unused observation $(s_{n+1},\pi(s_{n+1}),r,s)$ in $\mathcal O$ that could be used to extend the path by appending the observation. In the following, we say that such a path is \textit{non-extendible}.
Algorithm~\ref{alg2} provides an algorithm for constructing a non-extendible path. Alternative constructions could be used for obtaining non-extendible paths as well.

\begin{algorithm}[t]
\caption{Path Construction}\label{alg2}
\begin{algorithmic}[1]
\STATE \textbf{Input:} Observation sequence $\mathcal O$, policy $\pi$, initial state $s_1$\\ \smallskip
\STATE Set $t=1$ and let path ${\mathcal P_\pi}$ be empty.\\ \smallskip
\WHILE {$\mathcal O$ contains an unused observation of the form $(s_t,\pi(s_t),\cdot,\cdot)$}  
  \STATE Choose the first unused occurrence $o_t:= (s_t,\pi(s_t),r,s)$ of such an observation.  
  \STATE Append $o_t$ to ${\mathcal P_\pi}$.
  \STATE Mark $o_t$ in $\mathcal O$ as used.
  \STATE Set $s_{t+1}:=s$ and $t:=t+1$.
\ENDWHILE
\STATE Mark all observations in $\mathcal O$ as unused.
\STATE \textbf{Output:} sample path ${\mathcal P_\pi}$
\end{algorithmic}
\end{algorithm}

For each possible policy~$\pi$ the algorithm computes an estimate of the average reward $\rho_\pi$ from the sample path $\mathcal P_\pi$
and considers an optimistic upper confidence value $\tilde{\rho}_\pi$ (cf.~line~\ref{l:civ} of \osp) using the concentration results of Section~\ref{sec:prel}. 
The policy with the maximal $\tilde{\rho}_\pi$ is chosen for use in phase~$k$. 
The length $n_k$ of phase $k$, in which the chosen policy $\pi_k$ is used, depends on the length $n_{<k}:=|\mathcal P_{\pi_k}|$ of the sample path $\mathcal P_{\pi_k}$. That is, $\pi_k$ is usually played for $n_{<k}$ steps, but at least for $\sqrt{\frac{T}{SA}}$ steps (cf.~line~\ref{l:act}).

Note that at the beginning, all sample paths are empty in which case we set the confidence intervals to be $\infty$, and the algorithm chooses an arbitrary policy. The initial state of the sample paths can be chosen to be the current state, but this is not necessary. Note that by the Markov property the outcomes of all samples are independent of each other. The way Algorithm~\ref{alg2} extracts observations from $\mathcal O$ is analogous to when having access to a generative sampling model as e.g.\ assumed in work on sample complexity bounds like~\cite{pac}. In both settings the algorithm can request a sample for a particular state-action pair $(s,a)$. The only difference is that in our case at some point there are no suitable samples available anymore, when the construction of the sample path is terminated.

As the goal of this paper is to demonstrate an easy way to obtain optimal regret bounds, we do not elaborate in detail on computational aspects of the algorithm. A brief discussion is however in order. First, note that it is obviously not necessary to construct sample paths from scratch in each phase. It is sufficient to extend the path for each policy with new and previously unused samples. Further, while the algorithm as given is exponential in nature (as it loops over all $A^S$ policies), it may be possible to find the optimistic policy by some kind of optimistic policy gradient algorithm~\cite{lazaric}. 
We note that policies in ergodic MDPs exhibit a particular structure (see Section 3 of \cite{ort2}) that could be exploited by such an algorithm. However, at the moment this is not more than an idea for future research and the details of such an algorithm are yet to be developed.

\section{Regret Analysis}

The following theorem is the main result of this note.
\begin{theorem}\label{thm}
In uniformly ergodic MDPs, with probability at least $1-\delta$ the regret of \osp\ is bounded by
\[
  R_T \,\leq\, 4 \log(\tfrac{8T^2}{\delta}) \sqrt{\tm SA T},
\]
provided that $T\geq S^3 A\left(\frac{152 \tm \log \frac{8T^2}{\delta}}{\mu_{\min}^2}\right)^2$,
where $\mu_{\min}:=\min_{\pi,s:\mu_\pi(s)>0}\mu_\pi(s)$.
\end{theorem}

The improvement with respect to previously known bounds can be achieved due to the fact that
the confidence intervals for our algorithm are computed on the policy level and not on the level of rewards and transition probabilities as for UCRL \cite{ucrl}.
This avoids the problem of having rectangular confidence intervals that lead to an additional factor of $\sqrt{S}$ in the regret bounds for UCRL, cf.\ the discussion in \cite{osband}.

To keep the exposition simple, we have chosen confidence intervals which give a high probability bound for each horizon $T$. It is easy to adapt the confidence intervals to gain a high probability bound that holds for all $T$ simultaneously (cf. \cite{ucrl}).

The mixing time parameter in our bounds is different from the transition parameters in the regret bounds of \cite{ucrl1,ucrl} or the bias span used in \cite{regal,ronan}. We note however that for reversible Markov chains, $\tm$ is linearly bounded in the diameter (i.e., the hitting time) of the chain, cf. Section 10.5 of \cite{levin}. It follows from the lower bounds on the regret in \cite{ucrl} that the upper bound of Theorem~\ref{thm} is best possible with respect to the appearing parameters. Mixing times have also been used for sample complexity bounds in reinforcement learning~\cite{e3,rmax}, however not for a fixed constant $\frac14$ as in our case but with respect to the required accuracy.
It would be desirable to replace the upper bound $\tm$ on all mixing times by the mixing time of the optimal policy like in \cite{poladv}.
However, the technique of \cite{poladv} comes at the price of an additional dependence on the number of considered policies, which in our case  
obviously would damage the bound.

The parameter $T$ can be guessed using a standard doubling scheme getting the same regret bounds with a slightly larger constant. Guessing $\tm$ is more costly. Using $\log T$ as a guess for $\tm$, the additional regret is an additive constant exponential in $\tm$. We note however, that it is an open problem whether it is possible to get regret bounds depending on a different parameter than the diameter (such as the bias span) without having a larger bound on the quantity, cf.\ the discussion in Appendix~A of \cite{ronan2}.

\subsection{Proof of Theorem~\ref{thm}}

Recall that $\pi_k$ is the policy applied in phase $k$ for $n_k$ steps. The respective optimistic estimate $\tilde{\rho}_{\pi_k}$ has been computed from a sample path of length $n_{<k}$.

\subsubsection{Estimates $\tilde{\rho}_\pi$ are optimistic}
We start showing that the values $\tilde\rho_\pi$ computed by our algorithm from the sample paths of any policy $\pi$ are indeed optimistic. This holds in particular for the employed policies $\pi_k$.

\begin{lemma}\label{lem:conf}
With probability at least $1-\frac{\delta}{2}$, for all phases $k$ it holds that 
$\tilde\rho_{\pi_k} \geq \rho_{\pi_k}$.
\end{lemma}
\begin{proof}
Let us first consider an arbitrary fixed policy $\pi$ and some time step~$t$.
Using (column) vector notation $\bmu:=(\mu_{\pi}(s))_s$ and $\br:=(r(s,\pi(s))_s$ for the stationary distribution and the reward function under $\pi$, and writing $\bhmu$ and $\bhr$ for the respective estimated values at step $t$, we have 
\begin{eqnarray}\label{eq:err2}
      \rho_{\pi} - \hat{\rho}_\pi  &=&  \bmu^\top \br - \bhmu^\top \bhr
          \,=\,  (\bmu - \bhmu)^\top \br + \bhmu^\top (\br - \bhr). 
\end{eqnarray}
Let $n$ be the length of the sample path $\mathcal P_\pi$ from which the estimates are computed. 
Then the first term of \eqref{eq:err2} can be bounded by Corollary~\ref{cor-ci} as
\begin{eqnarray}\label{eq:err3}
    |(\bmu - \bhmu)^\top \br| \leq \sqrt{\frac{9 \tm \log \frac{8tT}{\delta}}{2n}}
\end{eqnarray}
with probability at least $1-\frac{\delta}{4T}$ (using a union bound over all $t$ possible values for $n$).
The second term of \eqref{eq:err2} can be written as
\begin{equation*}
  | \bhmu^\top (\br-\bhr) | = \tfrac{1}{n} \cdot \Big| \hspace{-6mm} \sum_{(s,\pi(s),r,s')\in\mathcal{P}_\pi}\hspace{-6mm}  (r(s,\pi(s)) - r) \Big|.
\end{equation*}
Since the sum is a martingale difference sequence, we obtain by Azuma-Hoeffding inequality (cf.\ Lemma A.7 in \cite{celu}) and another union bound that with probability $1-\frac{\delta}{4T}$
\begin{equation}\label{eq:err4}
  | \bhmu^\top (\br-\bhr) |  \leq \sqrt{\frac{\log\tfrac{8tT}{\delta}}{2n}}.
\end{equation}

Summarizing, we get from \eqref{eq:err2}--\eqref{eq:err4} that for any policy $\pi$ 
the estimate $\hat{\rho}_\pi$ computed at time step $t$ satisfies with probability at least $1-\frac{\delta}{2T}$ 
\begin{eqnarray*}\label{opt}
    {\rho}_\pi &\leq& \hat{\rho}_{\pi}  +  \sqrt{\frac{8\tm \log \tfrac{8tT}{\delta}}{n}}.
\end{eqnarray*}
Accordingly, 
writing $t_k$ for the time step when phase $k$ starts and $n_{\pi,t}$ for the length of the sample path for policy $\pi$ at step $t$,
\begin{equation*}
   \mathbb{P}\left\{ {\rho}_{\pi_k} > \hat{\rho}_{\pi_k}  +  \sqrt{\tfrac{8\tm \log \tfrac{8tT}{\delta}}{n_{\pi_k,t}}} \; \Bigg| \; \pi_k=\pi, t_k=t \right\} < \frac{\delta}{2T}.
\end{equation*}
It follows that
\begin{eqnarray*}
 \lefteqn{ \mathbb{P}\big\{ \exists k :  \tilde\rho_{\pi_k} < \rho_{\pi_k} \big\}  \leq \sum_k \mathbb{P}\big\{ \tilde\rho_{\pi_k} < \rho_{\pi_k} \big\} }\\
&\leq& \sum_k \sum_t \sum_\pi \mathbb{P}\left\{ {\rho}_{\pi_k} > \hat{\rho}_{\pi_k}  +  \sqrt{\tfrac{8\tm \log \tfrac{8tT}{\delta}}{n_{\pi_k,t}}} \; \Bigg| \; \pi_k=\pi, t_k=t \right\} \cdot  \mathbb{P}\big\{\pi_k=\pi, t_k=t \big\}\\
&\leq&  \sum_k   \sum_t \sum_\pi \frac{\delta}{2T} \cdot \mathbb{P}\big\{\pi_k=\pi, t_k=t \big\}\\
&=&  \sum_k \frac{\delta}{2T}  \sum_t \sum_\pi \mathbb{P}\big\{\pi_k=\pi, t_k=t \big\}\\
&=&  \sum_k \frac{\delta}{2T}\\
&\leq& \frac{\delta}{2}\,.
\end{eqnarray*}

\end{proof}

\subsubsection{Splitting regret into phases}
Lemma~\ref{lem:conf} implies that in each phase $k$ with high probability
\begin{equation*}\label{eq:optimism}
  \tilde{\rho}_{\pi_k} \,\geq \, \tilde{\rho}_{\pi^*} \,\geq\, \rho_{\pi_k}.
\end{equation*}
Accordingly, we can split and bound the regret as a sum over the single phases and obtain that with probability at least $1-\frac{\delta}{2}$, 
\begin{eqnarray}
   {R}_T &=&     T \rho^* - \sum_{t=1}^T r_t \,=\,  \sum_k n_k (\rho^* - \hat{\rho}_{\pi_k})  
      \,\leq\, \sum_k {n_k}(\tilde{\rho}_{\pi_k} -  \hat{\rho}_{\pi_k}) \nonumber \\
      &\leq&  \sum_k n_k \sqrt{\frac{8\tm \log \tfrac{8T^2}{\delta}}{n_{<k}}}. \label{eq:r0}
\end{eqnarray}
Now we can distinguish between two kinds of phases: The length of most phases is $n_k=n_{<k}$.
However, there there are also a few phases where the sample path for the chosen policy $\pi_k$ is shorter than $\sqrt{\frac{T}{SA}}$, when the length is $n_k=\sqrt{\frac{T}{SA}}>n_{<k}$. Let $\K^-:=\{k\,|\,n_k>n_{<k}\}$ be the set of these latter phases and set $K^-:=|\mathcal K^-|$.  
The regret for each phase in $\K^-$ is simply bounded by $\sqrt{\frac{T}{SA}}$, while for phases 
$k\notin\K^-$ we use\footnote{The final phase may be shorter than $n_{<k}$.} $n_k\leq n_{<k}$ to obtain from \eqref{eq:r0} that
\begin{eqnarray}\label{eq:r00}
    {R}_T &\leq&   K^- \sqrt{\frac{T}{SA}} + \sum_{k\notin\K^-} \sqrt{8 n_k \tm \log \tfrac{8T^2}{\delta}} 
\end{eqnarray}
with probability at least $1-\frac{\delta}{2}$.

It remains to bound the number of phases (not) in $\K^-$. A bound on $K^-$ obviously gives a bound on the first term in~\eqref{eq:r00}, while a bound on the number $K^+$ of phases not in $\K^-$ allows to bound the second term, as by Jensen's inequality we have due to $\sum_{k\notin \K^-} n_k \leq T$ that
\begin{equation}\label{eq:r2}
   \sum_{k\notin\K^-}  \sqrt{8 n_k \tm \log \tfrac{8T^2}{\delta}}  
        \leq  \sqrt{8 T K^+ \tm \log \tfrac{8T^2}{\delta}}.
\end{equation}

\subsubsection{Bounding the number of phases}
The following lemma gives a bound on the total number of phases that can be used as a bound on $K^-$ and $K^+$ to conclude the proof of Theorem~\ref{thm}.
\begin{lemma}\label{lem:episodes}
With probability at least $1-\frac{\delta}{2}$, the number of phases up to step~$T$ is bounded by 
\begin{equation*}
 K \leq SA \log_{\frac{4}{3}}\left(\tfrac{T}{SA}\right),
\end{equation*}
provided that $T\geq S^3 A\left(\frac{152 \tm \log \frac{8T^2}{\delta}}{\mu_{\min}^2}\right)^2$,
where $\mu_{\min}:=\min_{\pi,s:\mu_\pi(s)>0}\mu_\pi(s)$.
\end{lemma}
\begin{proof}
Let $n_{<k}(s,a)$ be the number of visits to $(s,a)$ before phase $k$. Note that the sample path for each policy $\pi$ in general will not use all samples of $(s,\pi(s))$, so that we also introduce the notation $n^\pi_{<k}(s)$ for the number of samples of $(s,\pi(s))$ used in the sample path of $\pi$ computed before phase~$k$. Note that by definition of the algorithm, sample paths are non-extendible, so that for each $\pi$ there is a state $s^-$ for which all samples are used,\footnote{In particular, this holds for the last state of the sample path.} that is, $n_{<k}(s^-,\pi(s^-))=n_{<k}^\pi(s^-)$.
We write $\hmu_{<k}$ and $\hmu_{k}$ for the empirical distributions of the policy $\pi_k$ in the sample path of and in phase~$k$, respectively.

Note that for each phase $k$ we have
\begin{eqnarray}
d_{TV}(\mu_{\pi_k},\hmu_{<k}) &\leq& \sqrt{\frac{38 S \tm \log \frac{8T^2}{\delta}}{n_{<k}}} \quad \mbox{ and
}\label{eq:conf-ep1} \\
  d_{TV}(\mu_{\pi_k},\hmu_k) &\leq& \sqrt{\frac{38 S \tm \log \frac{8T^2}{\delta}}{n_k}},\label{eq:conf-ep2}
\end{eqnarray}
each with probability at least $1-\frac{\delta}{4T}$ by Corollary~\ref{cor} and a union bound over all possible values of $n_k$ and $n_{<k}$, respectively. By another union bound over the at most $T$ phases, \eqref{eq:conf-ep1} and \eqref{eq:conf-ep2}
hold for all phases~$k$ with probability at least $1-\frac{\delta}{2}$.
In the following, we assume that the confidence intervals of \eqref{eq:conf-ep1} and \eqref{eq:conf-ep2} hold, so that all following results hold with probability $1-\frac{\delta}{2}$.

Each phase $k$ has length at least $n_k \geq\sqrt{\frac{T}{SA}}$. Consequently, if $T\geq$ $S^3 A\Big(\frac{152 \tm \log \frac{8T^2}{\delta}}{\mu_{\min}^2}\Big)^2$, then it is guaranteed by \eqref{eq:conf-ep2} that in each phase $k$ it holds that $d_{TV}(\mu_{\pi_k},\hmu_k) \leq \frac{\mu_{\min}}{2}\leq \frac{\mu_{\pi_k}(s)}{2}$ and therefore for each state $s$
\begin{equation}\label{eq:a}   
  \frac{\mu_{\pi_k}(s)}{2}  \, \leq \,  \hmu_k(s).  
\end{equation}
 
Now consider an arbitrary phase $k$ and let $s^-$ be the state for which $n_{<k}(s^-,\pi_k(s^-))=n^{\pi_k}_{<k}(s^-)$, so that
in particular $\hmu_{<k}(s^-)\,n_{<k}=n^{\pi_k}_{<k}(s^-)$. 
We are going to show that the number of visits to $(s^-,\pi_k(s^-))$ is increased by (at least) a factor $\frac{4}{3}$ in phase~$k$.
By \eqref{eq:conf-ep1}--\eqref{eq:a} and using that $n_k\geq n_{<k}$ we have 
\begin{eqnarray*}
\hmu_{<k}(s^-)\,n_{<k} &\leq& \mu_{\pi_k}(s^-)\, n_{<k} + \sqrt{38n_{<k} S \tm \log \tfrac{8T^2}{\delta}} \\
&\leq& 2 \hmu_k(s^-)\, n_{<k} + \sqrt{38 n_{<k} S \tm \log \tfrac{8T^2}{\delta}} \\
&\leq& 2 \hmu_k(s^-)\, n_k + \sqrt{38 n_k S \tm \log \tfrac{8T^2}{\delta}} \\
&\leq& 2 \hmu_k(s^-)\, n_k + \frac{\mu_{\pi_k}(s^-)}{2} n_k  \\
&\leq& 3 \hmu_k(s^-)\, n_k,
\end{eqnarray*}
so that abbreviating $a^-:=\pi_k(s^-)$
\[
   n_{<k+1}(s^-,a^-) \,=\, n_{<k}(s^-,a^-) + \hmu_{k}(s^-)\,n_k 
      \,\geq\,  \tfrac{4}{3} \, n_{<k}(s^-,a^-).
\]
Hence in each phase there is a state-action pair for which the number of visits is increased by a factor of $\frac{4}{3}$. This can be used to show that the total number of phases $K$ within $T$ steps is upper bounded as
\begin{equation}\label{eq:ep}
 K \leq SA \log_{\frac{4}{3}}\left(\tfrac{T}{SA}\right).
\end{equation}
The proof of \eqref{eq:ep} can be rewritten from Proposition 3 in \cite{ort}, with the only difference that the factor 2 is replaced by $\frac{4}{3}$. 
\end{proof}

Finally, combining \eqref{eq:r00}, \eqref{eq:r2}, and Lemma~\ref{lem:episodes}, using that $K^-,K^+ \leq K$, we obtain
that with probability at least $1-\delta$
\begin{eqnarray*}
   {R}_T &\leq&   K^- \sqrt{\frac{T}{SA}} + \sum_{k\notin\K^-} \sqrt{8 n_k \tm \log \tfrac{8T^2}{\delta}} \\
         &\leq&  K^- \sqrt{\frac{T}{SA}} + \sqrt{8 T K^+ \tm \log \tfrac{8T^2}{\delta}} \\
         &\leq&  \sqrt{SAT} \log_{\frac{4}{3}}\left(\tfrac{T}{SA}\right)  + \sqrt{8\tm SAT\log_{\frac{4}{3}}\left(\tfrac{T}{SA}\right) \log \big(\tfrac{8T^2}{\delta}\big)}  , 
\end{eqnarray*}
which completes the proof of the theorem. \qed

\section{Discussion and Conclusion}
While we were able to close the gap between lower and upper bound on the regret for uniformly ergodic MDPs, there are still quite a few open questions.
First of all, the concentration results we use are only available for uniformly ergodic Markov chains, so a
generalization of our approach to more general communicating MDPs seems not easy. An improvement over the
parameter~$\tm$ may be possible by considering more specific concentration results for Markov reward processes. These
might depend not so much on the mixing time than the bias span~\cite{ronan}. However, even if one achieves such bounds,
the resulting regret bounds would depend on the maximum bias span over all policies. Obtaining a dependence on the bias
span of the optimal policy instead seems not easily possible. Finally, another topic for future research is to develop
an optimistic policy gradient algorithm that computes the optimistic policy more efficiently than by an iteration over
all policies.

\paragraph{Acknowledgments.} 
The author would like to thank Ronan Fruit, Alessandro Lazaric, and Matteo Pirotta for discussion
as well as Sadegh Talebi and three anonymous reviewers for pointing out errors in a previous version of the paper. 

\small
\bibliographystyle{abbrv}

\end{document}